\documentclass[letterpaper, 10 pt, conference]{ieeeconf}  

\IEEEoverridecommandlockouts                              

\usepackage{graphicx} 
\usepackage{times} 
\usepackage{amsmath} 
\usepackage{amssymb}  
\usepackage{url}
\usepackage{cite}
\usepackage{enumerate}
\usepackage{bm}
\usepackage{cases}
\usepackage{mathtools}
\usepackage[caption=false, font=footnotesize]{subfig}
\usepackage{algorithmicx}
\usepackage{algorithm}
\usepackage{algpseudocode}
\usepackage{booktabs, ragged2e}
\usepackage{tabularx}
\usepackage{xcolor}
\usepackage[para]{threeparttable}

\makeatletter
\let\NAT@parse\undefined
\makeatother
\usepackage[hidelinks]{hyperref} 

\newcommand\given[1][]{\:#1\vert\:}
\newtheorem{theorem}{Theorem}[section]
\newtheorem{definition}{Definition}[section]
\DeclarePairedDelimiter\abs{\lvert}{\rvert}%
\DeclarePairedDelimiter\norm{\lVert}{\rVert}%
\DeclareMathOperator*{\argmin}{argmin}%
\DeclareMathOperator*{\argmax}{argmax}%

\newcolumntype{C}{>{\Centering\arraybackslash}X}
\newcolumntype{L}{>{\raggedright\arraybackslash}X}
\newcolumntype{R}{>{\raggedleft\arraybackslash}X}

\title{\LARGE \bf
Real-Time Trajectory Planning for Autonomous Driving with Gaussian Process and Incremental Refinement
}

\author{Jie Cheng$^1$, Yingbing Chen$^2$, Qingwen Zhang$^2$, Lu Gan$^2$ and Ming Liu$^{1,2}$
\thanks{$^1$Jie Cheng and Ming Liu are with the Department of Electronic 
        and Computer Engineering, the Hong Kong University of Science and 
        Technology, Clear Water Bay, Kowloon, Hong Kong SAR, China. \texttt{jchengai@connect.ust.hk}, \texttt{eelium@ust.hk}
} 
\thanks{$^2$Yingbing Chen, Qingwen Zhang and Lu Gan are with the System Hub, the HongKong University of Science and Technology (GZ), Guangzhou, China. \{\texttt{ychengz, qzhangcb, luganaa}\}\texttt{@connect.ust.hk}} 
}

\newcommand\mycolor{black}

\begin{document}

\maketitle
\thispagestyle{empty}
\pagestyle{empty}

\begin{abstract}

Real-time kinodynamic trajectory planning in dynamic environments is critical yet challenging for autonomous driving. 
In this paper, we propose an efficient trajectory planning system for autonomous driving in complex dynamic scenarios through iterative and incremental path-speed optimization.
Exploiting the decoupled structure of the planning problem, a path planner based on Gaussian process first generates a continuous arc-length parameterized path in the Fren\'{e}t frame,
considering static obstacle avoidance and curvature constraints. We theoretically prove that it is a good generalization of the well-known jerk optimal solution.
An efficient s-t graph search method is introduced to find a speed profile along the generated path to deal with dynamic environments.
Finally, the path and speed are optimized incrementally and iteratively to ensure kinodynamic feasibility. 
Various simulated scenarios with both static obstacles and dynamic agents verify the effectiveness and robustness of our proposed method.
Experimental results show that our method can run at 20 Hz.
The source code is released as an open-source package.

\end{abstract}


\section{Introduction}

\subsection{Motivation}

{\color{\mycolor}The trajectory planning system is crucial for autonomous driving tasks \cite{katrakazas2015review}}.
It needs to be robust and efficient, as well as able to handle static obstacles, dynamic traffic participants and various constraints (\textit{e.g.,} speed limits, kinodynamic feasibility) in real-time, which is challenging. 

Existing works can be categorized into two main classes: \textit{spatio-temporal planning} \cite{ajanovic2018_st_search, brito2019mpc, ding2019ssc_planner, he2021tdr} and \textit{path-speed decoupled planning} \cite{gu2015tunable, fan2018baidu, zhou2021dl_iaps}.
{\color{\mycolor}Spatio-temporal (s-t)} approaches mostly adopt a hierarchy scheme.
They first find an initial solution in s-t space by searching or sampling, and some may further construct convex subspaces based on the rough solution.
Then the preliminary result is improved by solving a receding-horizon optimization problem or fitting it with a parameterized curve.
The main advantage of these approaches is that they can consider spatial and temporal constraints simultaneously.
However, finding an initial guess in the spatio-temporal space is non-trivial, sometimes even harder than solving the optimization problem.
Approximations (\textit{e.g.,} increasing grid size or sample resolution, heuristics) have to be used to make this problem tractable in limited time.
As a result, these approaches are prone to local optimum.
Path-speed decoupled approaches break the high dimensional problem into two easier subproblems: first, planing a path; then, generating a speed profile along that path.
The benefits are two folds. 
Firstly, they do not require a preliminary result in the s-t domain.
Secondly, they usually enjoy higher efficiency and flexibility.
{\color{\mycolor}Nevertheless, they are generally much harder to guarantee kinodynamic feasibility due to the decomposition structure.}
\begin{figure}
\centering
\includegraphics[width=3.0in]{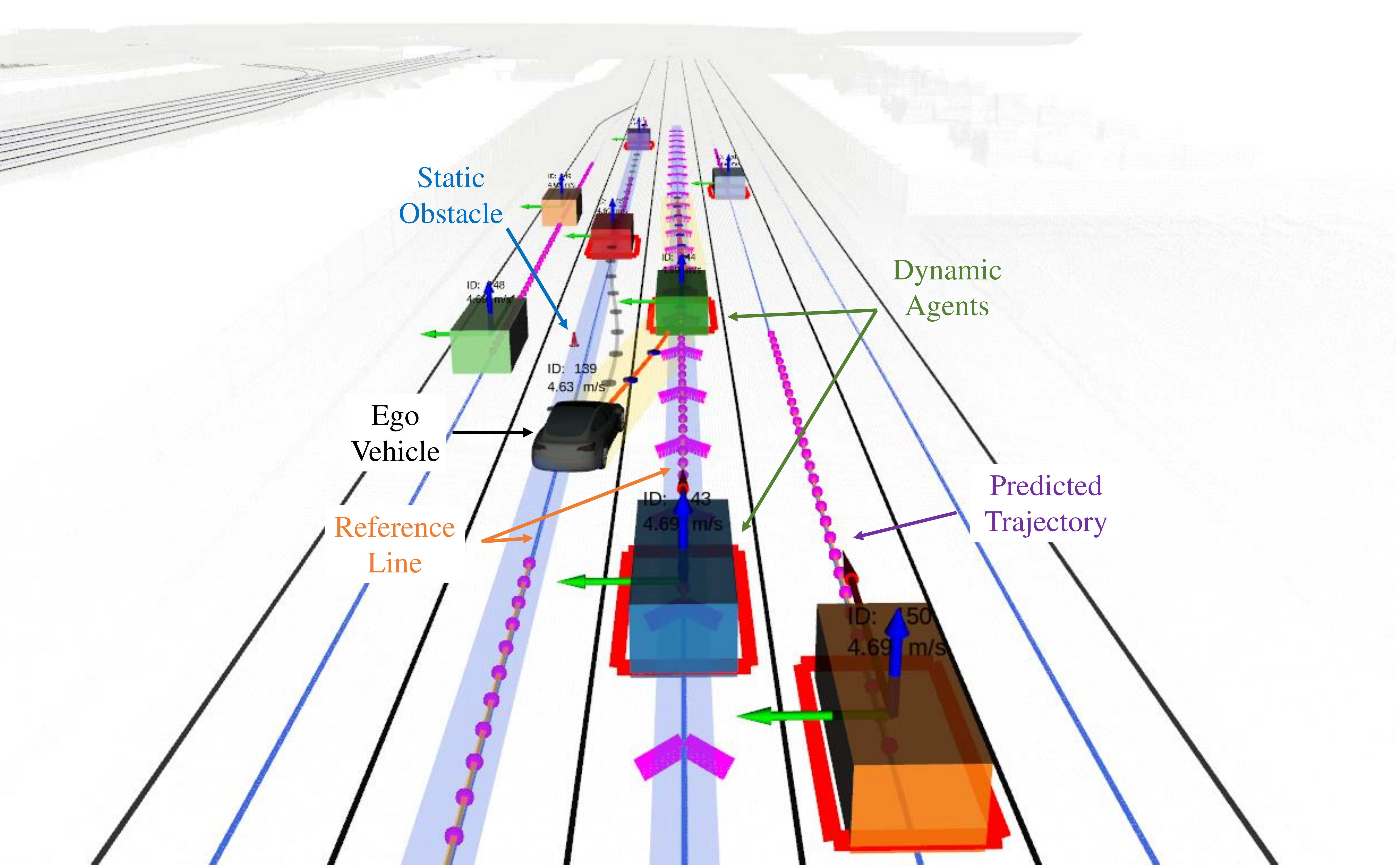}
\caption{Illustration of trajectory planning with both dynamical agents and a static obstacle.
The AV can either avoid the obstacle (gray trajectory) or merge into the right lane (red trajectory).
More examples can be found at \href{https://youtu.be/NHEZDrAzghI}{\color{\mycolor}{https://youtu.be/NHEZDrAzghI}.}}
\label{fig:intro}
\end{figure}

Apart from the issues mentioned above, most existing works have difficulty handling the curvature constraint.
Some ignore the constraint by assuming the resulting trajectory is smooth, which is not reasonable in some scenarios (\textit{e.g.,} sharp turns).
Others have to approximate or relax the constraint using a heuristic limit, leading to an overly conservative or even infeasible solution. 

In this paper, we propose an efficient trajectory planning framework to overcome the problems of existing methods.
Our approach can be categorized as path-speed decoupled planning in general.
Firstly, this path planning problem is converted to a probabilistic inference problem with Gaussian process (GP). 
{\color{\mycolor}This is motivated by recent progress on motion planning via probabilistic inference approach \cite{mukadam2016gaussian,mukadam2018continuous}, which shows significant increase in speed over other trajectory optimization solvers.}

Analytic formulation of curvature constraint is used and integrated into the path planning process. 
We also theoretically prove that this path generation method is a good generalization of the well-known jerk optimal solution \cite{werling2012optimal}.
To avoid local optimum, a novel and efficient s-t graph search method is introduced.
Finally, an incremental path-speed adjustment method is adopted to ensure kinodynamic feasibility.

\subsection{Contribution}
Compared with existing methods, our proposed approach is able to directly consider the curvature constraint without approximations, and efficiently generate high-quality kinodynamically feasible trajectories. 
The main contributions are summarized as follows:
\begin{itemize}
        \item A path planner incorporating GP that generates a collision-free path under strict curvature constraints.
        \item A complete and efficient trajectory planning framework for complex dynamic environments with a safety and kinodynamic feasibility guarantee.
        \item Comprehensive validation with various numerical experiments and open-source real-time implementation for the reference of the community.
\end{itemize}
\section{Related Work\label{related_works}}

\subsection{Iterative Path-speed Refinement}

The online iterative path-speed refinement technique is widely adopted in the unmanned aerial vehicle (UAV) community. 
For example, Zhou \textit{et al.} \cite{zhou2019robust} iteratively adjust the infeasible velocity and acceleration control points of the B-spline trajectory to achieve dynamic feasibility for a quadcopter,
{\color{\mycolor}while \cite{wang2020alternating} and \cite{sun2021fast} use a similar bilevel optimization approach to get spatio-temporal optimal trajectories in real-time.}
However, this technique is less commonly used for AVs. 
The main reason is that various non-linearities (\textit{e.g.}, collision constraints on a rectangular vehicle body, nonlinear vehicle dynamics) make the optimization problem for AVs expensive to solve iteratively. 
Adopting this approach for AVs,
Kapania \textit{et al.} \cite{kapania2016sequential} use a two-step algorithm to generate the trajectory for car racing with near real-time performance. 
{\color{\mycolor}Similarly, Tang \textit{et al.} \cite{tang2019time} apply bilevel optimization for time optimal car racing.} 
However, there are only applicable for static, obstacle-free racetracks.

{\color{\mycolor}Xu \textit{et al.} \cite{xu2021autonomous} use recurrent spline optimization to ensure curvature constraints, but do not consider dynamic feasibility when speed changes.} 
Recent work \cite{zhou2021dl_iaps} takes both dynamic feasibility and obstacle avoidance into consideration. 
However, the performance is not satisfactory (around 5Hz with the presence of obstacles).

{\color{\mycolor}By formulating the path planning as a probabilistic inference problem, we are able to partially improve the invalid path with efficient incremental inference tool \cite{kaess2012isam2} from the SLAM community, instead of generating a new path from scratch.}

\subsection{Motion Planning via Probabilistic Inference}
Recent progress on \textit{motion planning via probabilistic inference} has opened a new window onto the motion planning problem.
{\color{\mycolor}Mukadam \textit{et al.} \cite{mukadam2016gaussian} propose Gaussian process motion planning (GPMP), where they show GP driven by linear, time-varying stochastic differential equations (LTV-SDEs) is the appropriate tool connecting motion planning and probabilistic inference.}

The follow-up GPMP2 \cite{mukadam2018continuous} uses factor graphs to model the planning problem, and supports fast replanning by making use of an incremental Bayes tree solver \cite{kaess2012isam2}. 

{\color{\mycolor}Following them, we use Gaussian process for path planning}, 
and we theoretically investigate the connection between this planning method and the well-known jerk optimal solution \cite{werling2012optimal}.
\section{System Overview}

\begin{figure}
\vspace{0.5em}
\centering
\includegraphics[width=3.4in]{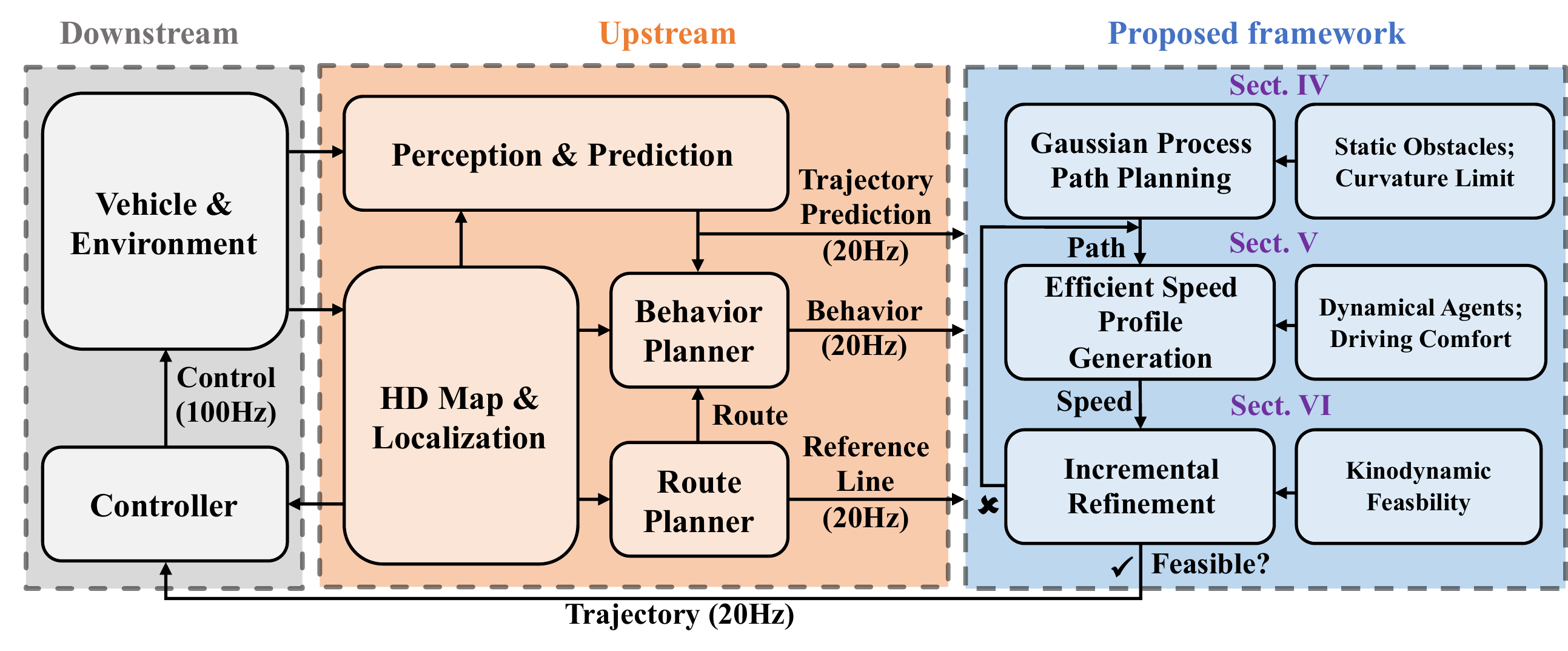}
\vspace{-1.5em}
\caption{Overview of the proposed trajectory planning framework and its relationship with other system components.}
\label{fig:overview}
\end{figure}

Fig. \ref{fig:intro} shows a typical scenario this paper tries to deal with.
Fig. \ref{fig:overview} outlines the proposed framework and its relationship with other system components.
It takes the trajectory prediction results, high-level behavior command and lane information as input, then generates collision free, executable trajectories.

The behavior means semantic level commands, such as ``lane change''.

The trajectory planning process contains three major steps as depicted in Fig. \ref{fig:overview} (blue part). 
First, A GP-based path planner generates a path in the Fren\'et frame.
Second, the speed planning module finds an optimal speed-profile along that path with the efficient s-t graph search.
Finally, when the result trajectory fails the feasibility check, we adjust the current path with the incremental refinement module. 
These three steps will be explained in Sect. \ref{sect:gaussian_process_path_planning}, \ref{sect:speed_planning} and \ref{sec:refinement} respectively.
\section{Gaussian Process Path Planning\label{sect:gaussian_process_path_planning}}

\subsection{GP Prior of the Fren\'et Path}
We consider a path in the Fren\'et frame $d(s)$, where $d$ is the lateral offset and $s$ is the arc-length along the reference line.
We define a prior distribution $\bm{d}(s) \sim \mathcal{GP}\left( \bm{\mu}(s), \bm{\mathcal{K}}(s,\tilde{s}) \right)$,
where $\bm{d}(s) = [d(s), d'(s), d''(s)]^\intercal$ is the lateral state, $\left(\cdot\right)'=\mathrm{d}(\cdot)/\mathrm{d}s$ denotes the derivative w.r.t. arc-length,
$\bm{\mu}(s)$ is the vector-valued mean function, and $\bm{\mathcal{K}}(s, \tilde{s})$ is a matrix-valued covariance function.
Given a set of longitudinal locations $\bm{s} = [s_0,\dots,s_N]^\intercal$, $\bm{d}$ has a joint Gaussian distribution:
\begin{align}
&\bm{d} \doteq \begin{bmatrix} \bm{d}_0 & \dots & \bm{d}_N \end{bmatrix}^\intercal \sim \mathcal{N} (\bm{\mu}, \bm{\mathcal{K}}), \label{eq:gp_prior}
\end{align}
where $\bm{\mu}$ is the mean vector, and $\bm{\mathcal{K}}$ is the covariance matrix , which are defined as
\begin{align}
&\bm{\mu} \doteq \begin{bmatrix} \bm{\mu}(s_0) & \dots & \bm{\mu}(s_N) \end{bmatrix}^\intercal,\\
&\bm{\mathcal{K}} \doteq \left[\bm{\mathcal{K}}(s_i, s_j)\right]_{ij,0\leq i,j \leq N}.
\end{align}
Following \cite{mukadam2016gaussian} and \cite{mukadam2018continuous}, $\bm{\mu}(s)$ and $\bm{\mathcal{K}}(s,\tilde{s})$ are generated by LTV-SDEs (arc-length as the parameter here):
\begin{align}
\bm{d}'(s) &= \mathbf{A}(s)\bm{d}(s) + \bm{u}(s) + \mathbf{F}(s)w(s),\label{ltv-sde} \\
w(s) &\sim \mathcal{GP}(0, \mathbf{Q}_C \delta(s-\tilde{s})),
\end{align}
where $\mathbf{A}(s)$ and $\mathbf{F}(s)$ are the arc-length varying matrices for the system, $\bm{u}(s)$ is the system input. $w(s)$ is the white process noise, $\mathbf{Q}_C$ is the power spectral matrix.
The mean and covariance function are generated by the solution of (\ref{ltv-sde}):
\begin{align}
\bm{\mu}(s) &= \mathbf{\Phi}(s,s_0)\bm{\mu_0} + \int_{s_0}^s\mathbf{\Phi}(s, \tau)\bm{u}(\tau)\mathrm{~d}\tau,\\
\begin{split}
\bm{\mathcal{K}}(s, \tilde{s}) &=
\mathbf{\Phi}(s,s_0)\bm{\mathcal{K}}_0\mathbf{\Phi}(\tilde{s},s_0)^\intercal \\
& \mkern-30mu + \int_{s_0}^{\text{min}(s,\tilde{s})}\mathbf{\Phi}(s,\tau)\mathbf{F}(s)\mathbf{Q}_C\mathbf{F}(s)^\intercal \mathbf{\Phi}(\tilde{s}, \tau)^\intercal \mathrm{~d}\tau,
\end{split}
\end{align}
where $\bm{\mu}_0$ and $\bm{\mathcal{K}}_0$ are the initial mean and covariance of $\bm{d}_0$, and $\mathbf{\Phi}(s, \tau)$ is the state transition matrix.
The prior distribution of the lateral states represented in terms of $\bm{\mu}$ and $\bm{\mathcal{K}}$ is
\begin{equation}
        P(\bm{d}) \propto \exp\left\{ - \frac{1}{2} \norm{\bm{d} - \bm{\mu}}_{\bm{\mathcal{K}}}^2 \right\}.
\end{equation}

This particular group of GPs generated by LTV-SDE has a special and helpful property: $\bm{d}(\tilde{s}),\, 0 \leq s_i < \tilde{s} < s_{i+1} \leq N$ only depends on its adjacent values $\bm{d}_i$ and $\bm{d}_{i+1}$.
$\bm{d}(\tilde{s})$ can be efficiently computed by
\begin{align}
\bm{d}(\tilde{s}) &= \bm{\mu}(\tilde{s}) + \mathbf{\Lambda}(\tilde{s})\left( \bm{d}_i - \bm{\mu}_i \right)\label{eq:interpolate}
 + \bm{\Psi}(\tilde{s})\left( \bm{d}_{i + 1} - \bm{\mu}_{i + 1} \right)\\ 
\bm{\Psi}(\tilde{s}) &= \mathbf{Q}_{\tilde{s}} \mathbf{\Phi}(s_{i + 1}, \tilde{s})^\intercal \mathbf{Q}_{i + 1}^{ - 1},\\
\mathbf{\Lambda}(\tilde{s}) &= \mathbf{\Phi}(\tilde{s}, s_i) - \bm{\Psi}(\tilde{s}) \mathbf{\Phi}(s_{i + 1}, s_i),\\
\mathbf{Q}_i &= \int_{s_{i - 1}}^{s_i} \mathbf{\Phi}(s_i, \tau) \mathbf{F}(\tau) \mathbf{Q}_C \mathbf{F}(\tau)^\intercal \mathbf{\Phi}(s_i, \tau)^\intercal \mathrm{~d}{\tau}. \label{eq:Q_i}
\end{align}

In the rest of this paper, we will use the white noise on jerk model (WNOJM) for (\ref{ltv-sde}), whose $\mathbf{\Phi}(s,\tau)$ and $\mathbf{Q}_i$ have a relatively simple form.
See Appendix \ref{wnoj} for details.

\subsection{Path Planning via Probabilistic Inference\label{sect:full_posterori}}
To illustrate how to convert the path planning into a probabilistic inference problem, we consider a simple case first.
Assume we have a linear “observer” $\bm{o}$ that observes the lateral states:
\begin{equation}
\bm{o}_i = \bm{d}_i  + \bm{n}_i, \quad \bm{n_i} \sim \mathcal{N}\left( \bm{0}, \mathbf{R}_i \right),
\end{equation}
where $\mathbf{R}_i$ is a positive definite diagonal matrix.
Together with (\ref{eq:gp_prior}), we have the joint Gaussian distribution:
\begin{equation}
p\left( \begin{bmatrix} \bm{d} \\ \bm{o} \end{bmatrix} \right) 
= \mathcal{N} 
\left( 
\begin{bmatrix} \bm{\mu} \\ \mathbf{M}\bm{\mu} \end{bmatrix},
\begin{bmatrix} 
\bm{\mathcal{K}} & \bm{\mathcal{K}} \mathbf{M}^\top \\
\mathbf{M}\bm{\mathcal{K}}^\top & \mathbf{M}\bm{\mathcal{K}}\mathbf{M}^\top + \mathbf{R}
\end{bmatrix}\right),
\end{equation}
where $\mathbf{M}$ is a matrix that makes the matrix size meet; $\bm{o}$ is a vector that contains $m$ observations,
\begin{equation}
\bm{o} = \begin{bmatrix} \bm{o}_0 & \dots & \bm{o}_m \end{bmatrix}^\intercal,\quad \mathbf{R} = \text{diag}\left( \mathbf{R}_1, \dots, \mathbf{R}_m \right);
\end{equation}
and the posterior is
\begin{equation}
\begin{aligned}
p(\bm{d} \given{\bm{o}}) &= \mathcal{N} \Big(
\left( \bm{\mathcal{K}}^{ - 1} + \mathbf{M}^\top \mathbf{R}^{ - 1} \mathbf{M} \right)^{ - 1} \left( \bm{\mathcal{K}^{ - 1}} \bm{\mu}\, \right. \\
&+ \left. \mathbf{M}^\top \mathbf{R}^{ - 1} \bm{o} \right) , \left( \bm{\mathcal{K}}^{ - 1} + \mathbf{M}^\top \mathbf{R}^{ - 1} \mathbf{M}  \right)^{ - 1} \Big).
\end{aligned} \label{eq:posterior}
\end{equation}
Suppose we want to plan a path from the initial state $\left( s_0, \bm{\bar{d}}_0 \right)$ to end state $\left( s_N, \bm{\bar{d}}_N \right)$.
Let $\bm{o} = [\bm{\bar{d}}_0, \bm{\bar{d}}_N]^\intercal$, $\mathbf{R}_0$ and $\mathbf{R}_N$ be a sufficiently small value.
The mean of the posterior (\ref{eq:posterior}) is the desired lateral state, and middle points are interpolated with (\ref{eq:interpolate}).
Next, we will prove that this path is actually the jerk optimal path given that $\bm{s} = \{s_0, s_N\}$.
\begin{definition}
A jerk optimal path $d^*(s)$ is a connection between the start state $\bm{\bar{d}}_0$ and an end state $\bm{\bar{d}}_N$ within the interval $S = s_N - s_0$ that minimizes the cost function:
\begin{equation*}
J = \int_{s_0}^{s_N} \big(d'''(s) \big)^2 \mathrm{~d}s.
\end{equation*}
\end{definition} 
\begin{theorem}\label{theo:jerk_optimal}
Let $\bm{s}=\{s_0, s_N\}$ and $\bm{o} = [\bm{\bar{d}}_0, \bm{\bar{d}}_N]^\intercal$.
Then the result path given by (\ref{eq:interpolate}) and (\ref{eq:posterior}) is the jerk optimal path connecting $\bm{\bar{d}}_0$ and $\bm{\bar{d}}_N$, provided that $\mathbf{R} \to \bm{0}$.
\end{theorem}
\begin{proof}
See Appendix \ref{proof}.
\end{proof}
Theorem \ref{theo:jerk_optimal} implies that the general path generated by (\ref{eq:posterior}) is actually piecewise quintic polynomials, and thus this approach is a good generalization of the known jerk optimal solution \cite{werling2012optimal}.

\subsection{Differentiable Path Constraints}

\begin{figure}
\centering
\vspace{0.5em}
\includegraphics[width=2.5in]{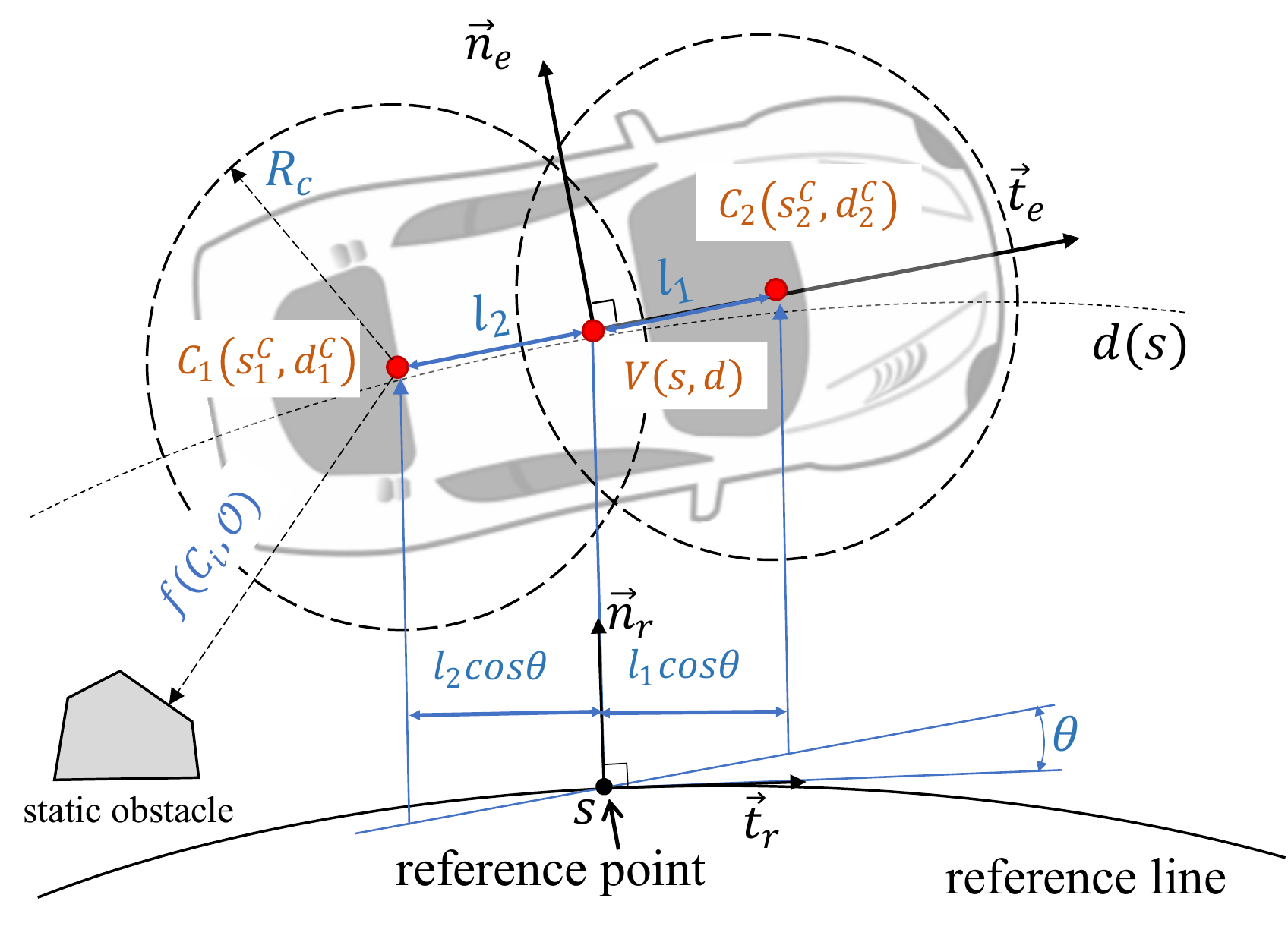}
\caption{An example of collision constraint with two bounding circles. $l_i$ is the distance from $V$ to $C_i$, $R_c$ is the radius of the circles and $\theta$ is the angle between $\vec{\bm{t}}_e$ and $\vec{\bm{t}}_r$.}
\label{fig:vehicle}
\end{figure}

Two major constraints need to be considered in path planning: collision avoidance and curvature constraints.

1) \textit{Collision Constraint}. As shown in Fig. \ref{fig:vehicle}, $V(s,d)$ is a path point of the desired path $d(s)$, and it normally represents the center of the rear axle or the mass of the AV.
To enforce the differentiable collision constraint, we cover the AV with several bounding circles $C_1, \dots, C_q$. 
We approximated the Fren\'et coordinates $(s^{C}_{i}, d^{C}_{i})$ of $C_i$ with
\begin{equation}
s^{C}_{i} = s + \xi l_i \cos \theta, \quad d^{C}_{i} = d + \xi l_i \sin \theta,
\end{equation}
where $\xi = \text{sgn}(\langle \overrightarrow{VC_i}, \overrightarrow{\bm{t}_e} \rangle)$ and $\theta$ is obtained by solving
\begin{equation}
        d' = [1 - \kappa_r d]\tan\theta,
\end{equation}
where $\kappa_r$ is the curvature of the reference point. 
Then the representation of collision constraint is
\begin{equation}
\min_{0\leq i \leq q} f \left( C_i, \mathcal{O}  \right) \geq R_c + \epsilon_s,
\end{equation}
where $\mathcal{O} = \{ \mathcal{O}_1, \dots, \mathcal{O}_k\}$ are static obstacles or road boundaries, $\epsilon_s$ is the safety margin, and $f$ is a function that returns the nearest distance to $\mathcal{O}$. We exploit the Euclidean Signed Distance Field (ESDF) for the distance query.
As depicted by Fig. \ref{fig:frenet}, the lane is discretized by a longitudinal resolution $s_\Delta$ and a lateral resolution $d_\Delta$. Static obstacles are projected w.r.t. the reference line.

\begin{figure}
\centering
\includegraphics[width=2.8in]{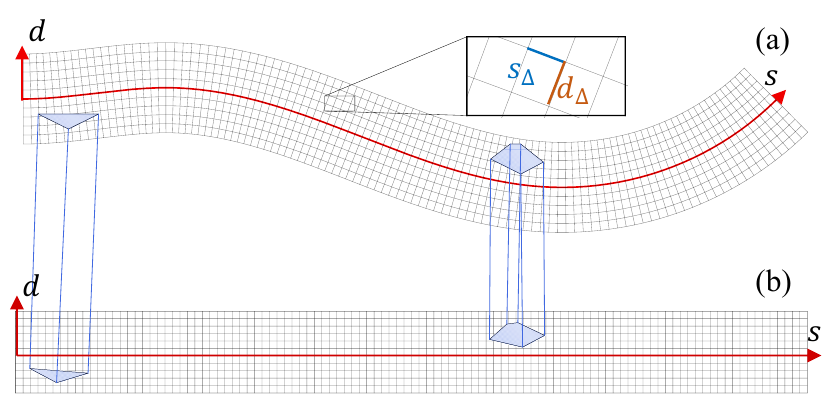}
\caption{Illustration of discretization of the lane w.r.t. the reference line (red line in the figure) and projection of the static obstacles.}
\label{fig:frenet}
\end{figure}

2) \textit{Curvature Constraint}. Constraining the maximum curvature of the path is crucial for large turning maneuver such as collision avoidance and sharp turning.
It would be risky if the planned path cannot be followed by the vehicle. 
However, this constraint is seldom addressed in existing literatures. 
Firstly, it is difficult to obtain explicit representation of the curvature using time-parametrized continuous curve \cite{ding2019ssc_planner}. 
Besides, when describe the path with a set of discrete points \cite{zhou2021dl_iaps}, one can only approximate the curvature using finite difference, which requires fine discretization.

The main advantage of the proposed continuous arc-length parametrized path is that we have the closed-form representation of the curvature of any path point in the path, which is a function of the lateral states.
From \cite{werling2012optimal}, we have
\begin{equation}
d'' = - [\kappa_r d]'\tan\theta + \frac{1 -\kappa_r d}{\cos^2\theta}\left[\kappa_p \frac{1 - \kappa_r d}{\cos\theta} - \kappa_r\right],
\end{equation}
where $\kappa_r$ and $\kappa_r'$ are the curvature and derivative of the curvature of the reference point. By rearranging terms, we have the explicit formulation of the curvature:
\begin{equation}
\kappa_p = \left[d'' - (\kappa_r' d + \kappa_r d')\tan \theta\right] \frac{\cos^3 \theta}{(1- \kappa_r d)^2} + \frac{\kappa_r\cos\theta}{1 - \kappa_r d}.
\end{equation}
Finally, the curvature constraint is represented as 
\begin{equation}
\abs{\kappa_p} \leq \kappa_{\max }
\label{eq:kappa_limit}
\end{equation}

\subsection{Maximum a Posteriori Solution}
Path constraints are modeled as the likelihood function $L\big(\bm{d}_i \given \mathcal{O} \big)$, which specifies the probability of being collision free $P(\min_{0\leq j \leq q} f(C_j, \mathcal{O}) \geq R_c +\epsilon_s \given \bm{d}_i)$ or within the curvature limit $P(\abs{\kappa_i} \leq \kappa_{\max} \given \bm{d}_i)$, given the current lateral state $\bm{d}_i$.
We use a distribution of the exponential family to represent the likelihood:
\begin{equation}
L\big(\bm{d}_i \given \mathcal{O} \big) \propto \exp \left\{ - \frac{1}{2} \norm{\bm{p}(\bm{d}_i)}^2_{\bm{\Sigma}} \right\},
\end{equation}
where  $\bm{p}(\bm{d}_i)$ is a vector-valued cost function and $\bm{\Sigma} = \sigma \bm{I}$ is a hyperparameter. For the likelihood function of the collision constraint $L_{\text{col}}$, $\bm{p}_{\text{col}}$ is defined as
\begin{equation}
\bm{p}_{\text{col}} = \left[j_{\text{col}}\big(f(C_i, \mathcal{O}), R_c + \epsilon_s \big)\right]\Big|_{1\leq i \leq q},
\end{equation}
where $j_{\text{col}}:R\mapsto R$ is a second-order continuous penalty function (the continuity can be checked by directly taking derivatives) given by
\begin{equation}
j_{\text{col}}(x, x_r)= 
\begin{cases}
0, & (\bar{x} \leq 0)\\
\alpha \bar{x}^3, & (0 < \bar{x} \leq x_r)\\
\alpha (3 x_r \bar{x}^2 - 3 x_r^2 \bar{x} + x_r^3), & ( \bar{x} > x_r)
\end{cases},
\end{equation}
where $\bar{x} = x_r - x$, $x_r$ is a parameter that specifies the boundary, and $\alpha$ is a normalization factor.
For the likelihood function of curvature constraint $L_{\text{cur}}$, $\bm{p}_{\text{cur}}$ is defined similarly:
\begin{equation}
\bm{p}_{\text{cur}} = \left[j_{\text{cur}}\big(\kappa_p(\bm{d_i}), \kappa_{\max}\big)\right],
\label{eq:curvature_constraint}
\end{equation}
where $j_{\text{cur}}$ is basically the same as $j_{\text{col}}$, except that it is two-side bounded since the curvature is allowed to be in the range $\left[ -\kappa_{\max}, \kappa_{\max} \right]$.

With path constraints, it is unlikely to obtain a full posterior as Sect. \ref{sect:full_posterori} does. 
Fortunately, for path planning, we only care about the maximum a posteriori (MAP) solution:
\begin{equation}
\begin{aligned}
\bm{d}^* 
&= \underset{\bm{d}}{\argmax} \bigg\{ P(\bm{d})\prod_i \prod_{j = \text{col,cur}} L_j(\bm{d}_i \given \mathcal{O})\bigg\}\\ 
&=\underset{\bm{d}}{\argmin} \bigg\{ - \log \Big(  P(\bm{d})\prod_i \prod_{j = \text{col,cur}} L_j(\bm{d}_i \given \mathcal{O}) \Big) \bigg\}\\
&=\underset{\bm{d}}{\argmin} \bigg\{ \frac{1}{2} \norm{\bm{d} - \bm{\mu}}^2_{\bm{\mathcal{K}}} + \frac{1}{2} \norm{\bm{p}_{\text{col}}(\bm{d})}^2_{\bm{\Sigma}_{\text{obs}}} \\
&\qquad \qquad \quad+ \frac{1}{2}\norm{\bm{p}_{\text{cur}}(\bm{d})}^2_{\bm{\Sigma}_{\text{cur}}} \bigg\},
\label{eq:map}
\end{aligned}
\end{equation}
where $\bm{p}_{(\cdot)}(\bm{d}) = \left[ \bm{p}_{(\cdot)}(\bm{d}_0),\dots, \bm{p}_{(\cdot)}(\bm{d}_N) \right]^\intercal$.

\section{Efficient Speed-Profile Generation\label{sect:speed_planning}}

\begin{figure*}[!t]
\centering
\subfloat[Expansion and local trauncation]{\includegraphics[width=2.15in]{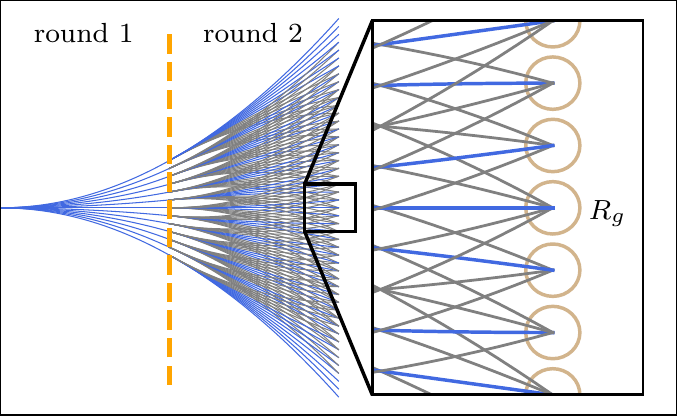}
\label{fig:local_trunc}}
\hfil
\subfloat[Overtaking scenario]{\includegraphics[width=2.15in]{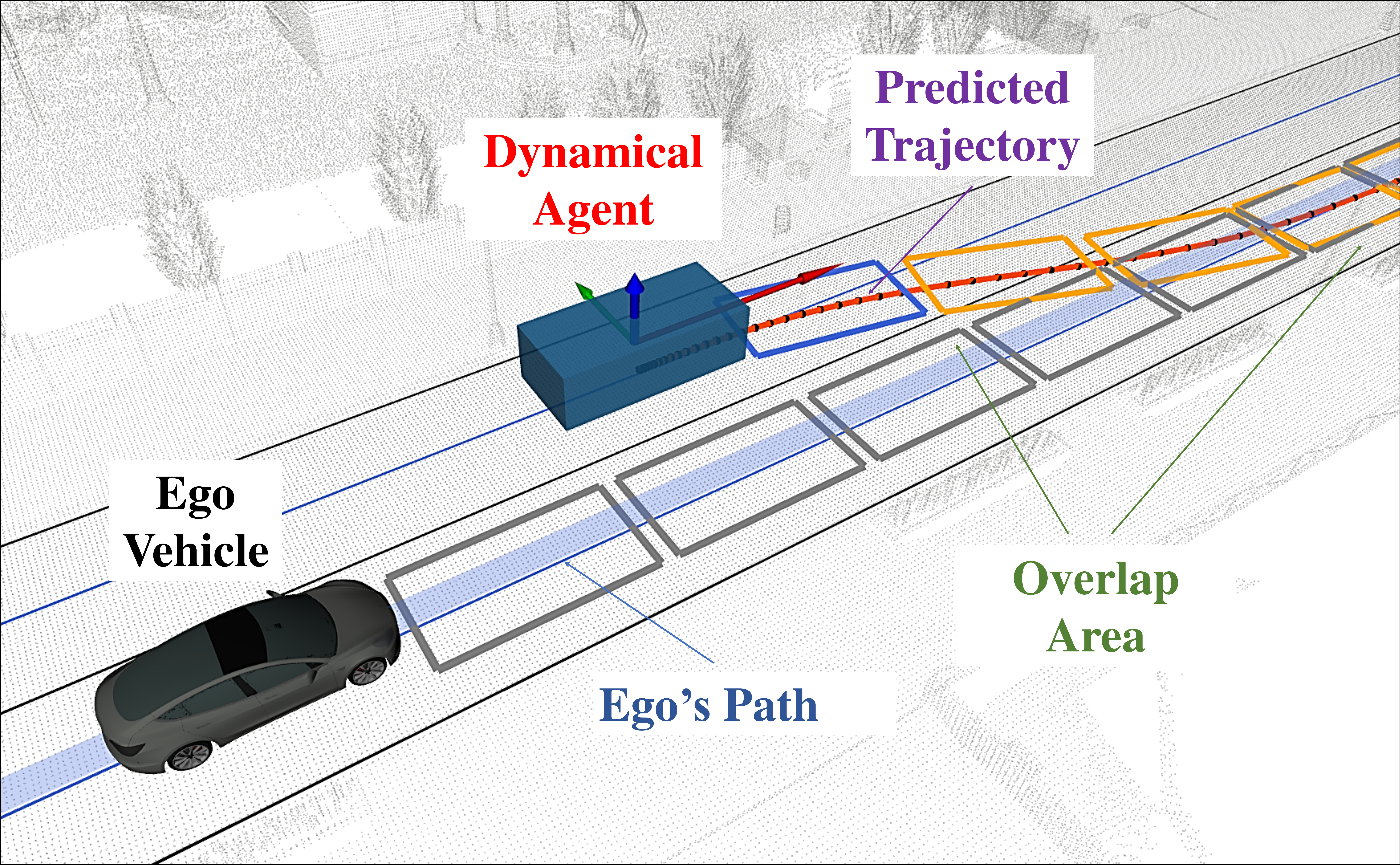}
\label{fig:overtake}}
\hfil
\subfloat[s-t graph of the overtaking scenario]{\includegraphics[width=2.15in]{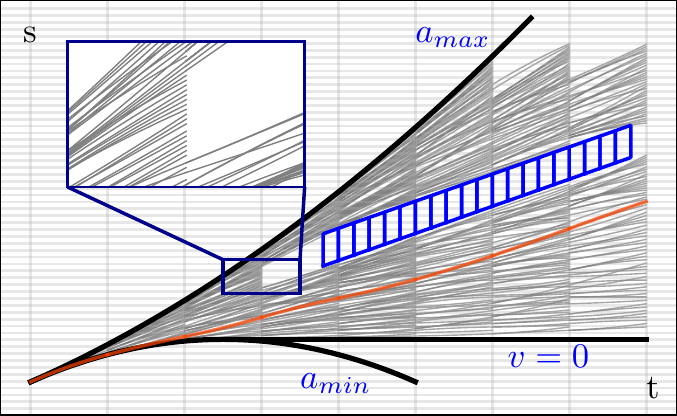}
\label{fig:overtake_st}}
\caption{Illustration of the efficient s-t graph search algorithm. 
(a) Results of forward expansion for two rounds. Gray lines represent the discarded states during local truncation.
(b) An illustration scenario where an agent is trying to overtake the AV. 
(c) The corresponding s-t search results. The blue area in the figure represents the blocked region, the gray lines are the intermediate results of the s-t graph search, and the red line is the optimal speed-profile.
}
\label{fig:st_graph_algorithm}
\end{figure*}

In this section, we assume that the AV's acceleration limit are $a_{\min}$ and $a_{\max}$, the initial longitudinal state is $\bm{s}_0 = [s_0, \dot{s}_0, \ddot{s}_0]^\intercal$, predicted future trajectories of dynamical agents are available, and a desired path generated by Sect. \ref{sect:gaussian_process_path_planning} is already given.
With a little abuse of notation, $s$ in this section represents the arc-length along the desired path. 
  
\subsection{Efficient S-t Graph Search}
In speed planning, local minimums inevitably exist. For example, in a merging scenario, the AV can choose to either decelerate to yield or accelerate to overtake the human driver.
To escape from the local minimum and ensure safety, a searching, or sampling step is necessary.
As outlined in Algorithm \ref{alg:efficient_search}, our proposed s-t graph search algorithm consists of three main steps: trajectory projection, forward expansion, and local truncation.

1) \textit{Trajectory Projection}.
A dynamical agent is critical if its future trajectory has possible intersections with the AV's desired path.
For each critical agent, a sequence of trajectory points are sampled within time interval $T_s$.
Then we check at every trajectory point whether the bounding box of the agent at that position has an overlap with the desired path.
If so, a blocked region is marked in the s-t graph, by finding its projection w.r.t. the desired path.

2) \textit{Forward Expansion}.
Starting with the initial state $\bm{s}_0$, we perform forward simulation with a set of discretized control inputs $\mathcal{A} = \{ a_{\min}\leq a_1, \dots, a_k \leq a_{\max}\}$ for time interval $T_f$ using the constant acceleration model:
\begin{equation}
\begin{bmatrix} s_{i + 1}^j \\ \dot{s}_{i + 1}^j  \end{bmatrix} = 
\begin{bmatrix} s_i + \dot{s}_i T_f + \frac{1}{2} a_j T_f^2 \\ \dot{s}_i + a_j T_f  \end{bmatrix}
, 1 \leq j \leq k.
\end{equation}
During the expansion, any child state, which falls into the blocked regions, will be abandoned.
After expansion, each survived child state will repeat that process until reaching the maximum planning duration or path length. 
Each state has a corresponding cost, which is defined as
\begin{equation}
J = J_p +  w_u J_u + w_r J_r + w_b J_b,
\end{equation}
where $J_p$ is the cost of the parent state, $J_u = \int_0^{T_f} a^2 dt$ is the control cost, $J_r = \abs{\dot{s} - \dot{s}_{\text{ref}}}$ is the deviation from the reference velocity, $J_b$ is the cost of being close to blocked regions, and $w_{(\cdot)}$ are weights.

3) \textit{Local Truncation}.
As shown in Fig. (\ref{fig:local_trunc}), the number of states will grow exponentially during the expansion, which would be problematic for real-time implementation.
To reduce the computation costs, we propose a local truncation mechanism.
In the $j$-th round of expansion, the longitudinal position of an arbitrary child state is
\begin{equation}
s_j = s_0 + j \dot{s}_0 T_f + \frac{1}{2} \underbrace{\left[(2j - 1) a_{j}^0 +\dots + a_{j}^j \right]}_{\overline{a}} T_f^2,
\label{eq:child_node_location}
\end{equation} 
where $a_j^i|_{0\leq i \leq j} \in \mathcal{A}$ is the control input in the $i$-th round.
(\ref{eq:child_node_location}) indicates that child states with similar $\overline{a}$ are near to each other in the s-t graph, thus contributing little to the exploration of the free space.
Based on this observation, in each round, the child states within radius $R_g$ will be clustered into groups.
In each group, only the state with the minimum cost will be used for the next round.
This can be viewed as the balance between exploration and exploitation.
As Fig. (\ref{fig:local_trunc}) depicted, the number of states is largely reduced after the local truncation. 

Fig. (\ref{fig:overtake}) shows a scenario where an agent is trying to overtake the AV, and Fig. (\ref{fig:overtake_st}) shows the corresponding s-t search results.
The search result indicates that the AV decides to yield to the human driver.

\begin{figure}[!t]
\vspace{-0.7em}
\begin{algorithm}[H]
\caption{Efficient s-t graph search}\label{alg:efficient_search}
\begin{algorithmic}[1]
  \State \textbf{Notation}: a given path $\mathcal{P}$, predicted trajectories $\bm{\Gamma}$, s-t graph $\mathcal{G}$, states $\mathcal{V}$, forward steps $N_f$
  \State \textbf{Initialize}: $\mathcal{G} \gets \emptyset$, $\mathcal{V} \gets \emptyset$
  \State $\mathcal{G} \gets$ \textsc{ProjectTrajectory}($\mathcal{P}$, $\bm{\Gamma}$), $\mathcal{V}_0 \gets$ \textsc{Node}($\bm{s}_0$)
  \For{$i \in \{1, \dots, N_{f}\}$}
        \For{$v \in \mathcal{V}_{i-1}$}
        \State $\overline{\mathcal{V}} \gets \emptyset$ 
        \For{$a \in \{a_1,\dots,a_k\}$}
                \State $\overline{v} \gets$ \textsc{ForwardExpansion}($v,a, T_f$)
                \If{\textsc{NotBlocked}($\mathcal{G}, \overline{v}$)}
                    \State \textsc{CalculateCost}($\mathcal{G}, \overline{v}$)
                    \State $\overline{\mathcal{V}} \gets \overline{v}$
                \EndIf
        \EndFor
        \EndFor
        \State $\mathcal{V}_i \gets$ \textsc{LocalTruncation}($\overline{\mathcal{V}}, R_g$)
  \EndFor
\end{algorithmic}
\end{algorithm}
\vspace{-2em}
\end{figure}

\subsection{S-t Curve Smoothing\label{sec:smoothing}}
Although we efficiently find a speed-profile through Algorithm \ref{alg:efficient_search}, we implicitly assume that the acceleration changes instantaneously during the forward expansion.
To improve the driving comfort, a smoothing step is desired.
We fit the coarse speed-profile with a piecewise B\'ezier curve and construct the overall problem as a quadratic programming (QP).
This approach has been extensively studied, 

and we refer interested readers to \cite{ding2019ssc_planner} for details.

Particular attention should be paid to the speed constraint. The maximum speed is determined by two elements.
One is the speed limit of the road $v_{\text{road}}$. The other is derived from the lateral acceleration constraint:
\begin{equation}
        v(t) \leq v_\kappa^{\max}(t) = \sqrt{a_{\text{lat}}^{\max} / \kappa(s(t))},\label{eq:l_acc_constraint}
\end{equation}
where $a_{\text{lat}}^{\max}$ is the maximum lateral acceleration, and $\kappa(s(t))$ is the curvature of the path.
However, since we do not know $s(t)$ in advance, we can only estimate $v_{\kappa}^{\max}(t)$ using the initial speed-profile.
This approximation is sometimes problematic, and we will discuss it further in Sect. \ref{sec:refinement}. 

\section{Incremental Refinement\label{sec:refinement}}

As aforementioned in Sect. \ref{sec:smoothing}, the final trajectory may violate the lateral acceleration constraint if we overestimate (\ref{eq:l_acc_constraint}).
Existing solutions mostly adopt an iterative optimization approach, which is computationally expensive.
Actually, since violations are usually sparse and local, one would naturally ask whether it is necessary to resolve the whole problem in each iteration as the majority of the trajectory is valid and unchanged.
Motivated by \cite{mukadam2018continuous}, we propose to iteratively refine the trajectory in an incremental manner.

\subsection{Factor Graph and Bayes Tree}
 A \textit{factor graph} \cite{factorgraph} is a bipartite graph that represents the factorization of a function.
It is known that performing inference on a factor graph is equivalent to solving the MAP problem (\ref{eq:map}).
Fig. \ref{fig:factor_graph}a shows an example factor graph of the MAP problem (\ref{eq:map}), where factors represent priors and likelihood functions.
Through variable elimination, a factor graph is converted to a \textit{Bayes tree} \cite{kaess2010bayes}, which is an efficient tool for incremental inference.
Fig. \ref{fig:factor_graph}b shows the corresponding Bayes tree of Fig. \ref{fig:factor_graph}a with an elimination order from the first state to the last state.
One important property is that changes in the factor graph only affect parts of the Bayes tree. See \cite{kaess2012isam2} for more details.

\subsection{Incremental Trajectory Refinement}
As outlined in Algorithm \ref{alg:refinement}, in the initialization step, the path and speed profile are generated using the methods described in Sect. \ref{sect:gaussian_process_path_planning} and Sect. \ref{sect:speed_planning} respectively.
In each iteration, $n$ trajectory points are sampled at time stamp $\{t_0,\dots,t_n\}$.
If the lateral acceleration of the trajectory point $i$ exceeds the allowable value $a_{\text{lat}}^{\max}$, a new factor is added to the factor graph, which constrains the maximum lateral acceleration of the path at $s(t_i)$.
The lateral acceleration is estimated using the speed profile of the last iteration with
\begin{equation}
        a_{\text{lat}}\left( s(t_i) \right) = d''(s(t_i))\dot{s}(t_i)^2 + d'(s(t_i))\ddot{s}(t_i),
\end{equation}
where $\bm{d}(s(t_i))=[d(s(t_i)), d'(s(t_i)), d''(s(t_i))]$ is interpolated using (\ref{eq:interpolate}) with adjacent states. 

When the Bayes tree updates, only the parts of path associated with the newly added factors is updated.
Thereafter, we regenerate the speed profile for the new path. This process is repeated until all constraints are satisfied or the maximum number of iterations is reached.

\begin{figure}
\centering
\includegraphics[width=3.0in]{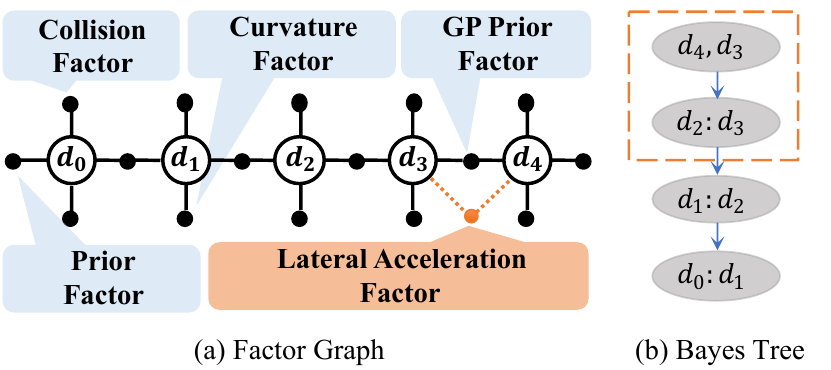}
\caption{An example factor graph and Bayes tree.
When a lateral acceleration factor is added to the graph, only part of the Bayes tree (dashed orange box in (b)) will be affected.}
\label{fig:factor_graph}
\end{figure}

\begin{figure}
\begin{algorithm}[H]
\caption{Incremental trajectory refinement}\label{alg:refinement}
\begin{algorithmic}[1]
  \State \textbf{Initialize}: $\mathcal{P}_0 \gets \textsc{GPPathPlanning}(\bm{s}, \mathcal{O})$\\ $\qquad \quad \quad\; \mathcal{S}_0 \gets \textsc{SpeedProfileGeneration}(\mathcal{P}_0,\Gamma)$ 
  \State \textit{BayesTree.Init}$(\mathcal{P}_0)$
  \For{k in $\{1,\dots,\text{max\_iter}\}$}
  \State $f_{\text{add}} \gets \emptyset$
        \For{$t_i$ in $\{t_0,\dots,t_n\}$}
        \State $a_{\text{lat}}(t_i) \gets \textsc{LateralAcc}(\mathcal{P}_{k-1}, \mathcal{S}_{k-1}, t_i)$
        \If{$a_{\text{lat}}(t_i) > a_{\text{lat}}^{\max}$}
             \State $f_{\text{add}} \gets \textsc{NewFactor}\left(\bm{d}(s(t_i)), \mathcal{S}_{k-1}, t_i\right)$
        \EndIf
        \EndFor
        \If{$f_{\text{add}} == \emptyset$}
                \textbf{break}
        \EndIf
        \State $\mathcal{P}_k \gets \textit{BayesTree.Update}(f_{\text{add}})$
        \State $\mathcal{S}_k \gets \textsc{UpdateSpeedProfile}(\mathcal{P}_k)$
  \EndFor
\end{algorithmic}
\end{algorithm}
\vspace{-1em}
\end{figure}
\section{Experimental Results}

\def\CC{{C\nolinebreak[4]\hspace{-.05em}\raisebox{.4ex}{\tiny\bf ++}}}

\subsection{Implementation Details}
We initialize the GP path planner by solving a simplified problem where constraints are ignored.
This step is necessary for almost all path planners, since we have to decide on which side (left or right) to pass by the obstacles.
The total length of the planning path is $100\,\mathrm{m}$ and we uniformly sample $N=20$ longitudinal locations along the reference line.
Between $\bm{d}_i$ and $\bm{d}_{i+1}$, we add 10 more interpolated lateral states using (\ref{eq:interpolate}) for collision and curvature constraints.
The ESDF ($8\,\mathrm{m} \times 100\,\mathrm{m}$) is built with a resolution $d_\Delta = s_\Delta = 0.1\,\mathrm{m}$ in each planning cycle.
For speed planning, the control limits of the AV are set to $a_{\max}=2\,\mathrm{m/s}$ and $a_{\min} = -4\,\mathrm{m/s}$.
We use $k=13$, and $T_f = 1\,\mathrm{s}$ in s-t the graph search.
The proposed method\footnote{Source code will be released at \href{https://github.com/jchengai/gpir}{\color{blue}https://github.com/jchengai/gpir}} is implemented in C\texttt{++}, and we use GTSAM \cite{dellaert2012factor} to implement Algorithm \ref{alg:refinement}.
All the experiments are conducted on a desktop computer with an Intel I7-8700K CPU (3.7 GHz).

\subsection{Qualitative Results}

\begin{figure}
   \centering
   \vspace{0.5em}
   \subfloat[t=11.60s, adding obstacle to lane]{\includegraphics[width=3.2in]{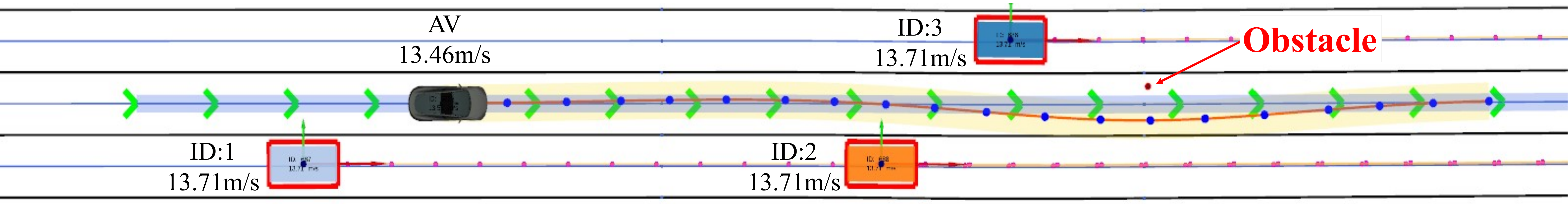}
      \label{fig:oads1}}\\
   \vspace{-0.5em}
   \subfloat[t=13.86s, decelerating to yield]{\includegraphics[width=3.2in]{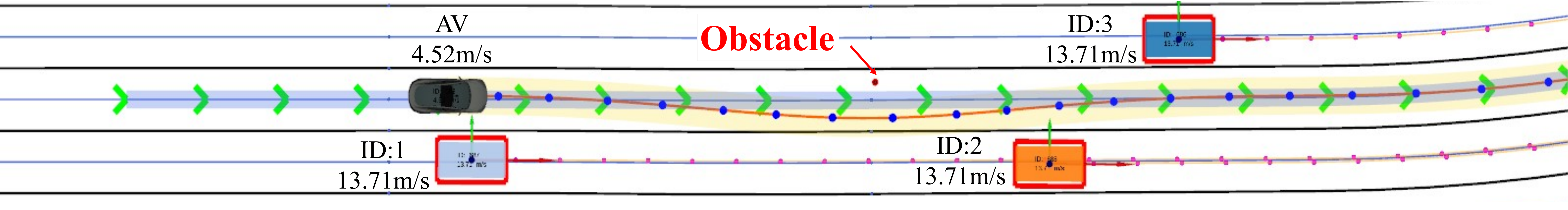}
      \label{fig:oads2}}\\
   \vspace{-0.5em}
   \subfloat[t=17.04s, accelerating to pass]{\includegraphics[width=3.2in]{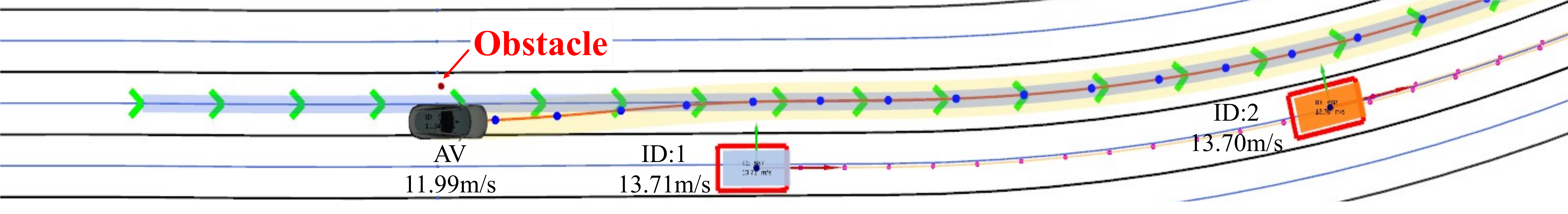}
      \label{fig:oads3}}\\
   \vspace{-0.5em}
   \subfloat[t=11.60s]{\includegraphics{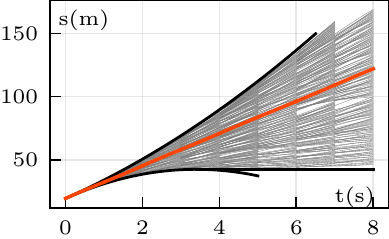}}
   \hfil
   \subfloat[t=13.86s]{\includegraphics{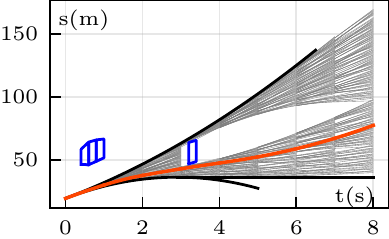}}
   \caption{Illustration of obstacle avoidance in dynamic environments. (d) and (e) show the corresponding s-t graphs of (a) and (b).}
   \label{fig:oads}
\end{figure}

We validate our method in the CARLA \cite{carla} simulation environment.
The predicted trajectories are generated by the constant velocity model for simplicity, but our method does not depend on any particular prediction module.

Fig. \ref{fig:oads} depicts a scenario where the AV has to avoid an obstacle in front of it while there are other dynamical agents around.
The AV decelerates and yields to other agents to avoid a potential risky interaction near the obstacle.
This case shows our proposed path planner's capability of obstacle avoidance and effectiveness of the s-t graph search algorithm.
More illustrative scenarios, including consecutive obstacle avoidance, navigating in dense traffic, lane changing at high speed and {\color{\mycolor}sharp turning manuever} can be found in the supplementary video at \href{https://youtu.be/NHEZDrAzghI}{\color{blue}{https://youtu.be/NHEZDrAzghI}}.

\begin{table}[t]
   \begin{threeparttable}
      \caption{Time Consumption of Different Planning Steps in One Cycle}
      \centering
      \label{tab:perfomance}
      \setlength\tabcolsep{3pt} 
      \begin{tabularx}{\columnwidth}{@{}lCC cc cc@{}}
         \toprule
                  & \multicolumn{4}{c}{Planning Steps}           &      &                                                                                      \\
         \cmidrule(lr){2-5}
         Time
                  & \multicolumn{1}{C}{GP Path Planning}
                  & \multicolumn{1}{C}{Speed Profile Generation}
                  & \multicolumn{2}{C}{Iterative Refinement}
                  & \multicolumn{2}{c}{Total}                                                                                                                  \\
         \midrule
         Min(ms)  & 2.36                                         & 1.10 & \textbf{2.03}\tnote{a}  & 5.56\tnote{b}  & \textbf{6.06}\tnote{a}  & 10.09\tnote{b}  \\
         Avg.(ms) & 6.09                                         & 3.77 & \textbf{6.41}\tnote{a}  & 16.47\tnote{b} & \textbf{17.22}\tnote{a} & 27.28\tnote{b}  \\
         Max.(ms)  & 22.78                                        & 6.99 & \textbf{29.96}\tnote{a} & 71.89\tnote{b} & \textbf{57.73}\tnote{a} & 101.66\tnote{b} \\
         \bottomrule
      \end{tabularx}

      \smallskip
      \scriptsize
      \begin{tablenotes}
         \RaggedRight
         \item[a] Incremental version
         \item[b] Non-incremental version
      \end{tablenotes}
   \end{threeparttable}
\end{table}

\begin{table}[h!]
   \begin{threeparttable}
      \caption{\color{\mycolor}Results For Random Planning Tasks}
      \centering
      \label{tab:my-table}
      \setlength\tabcolsep{5pt} 
      \begin{tabularx}{\columnwidth}{@{}cccc@{}}
         \toprule
                                    & G3P-Cur (ours)       & DL-IAPS\cite{zhou2021dl_iaps} & TDR-OBCA\cite{he2021tdr} \\
         \midrule
         Success\tnote{a} Rate (\%) & \textbf{98.90} & 55.40   & 97.40     \\
         Avg. Time (ms)             & \textbf{43.78}\tnote{b} & 154.15  & 1466.32   \\
         Max. Time (ms)             & \textbf{55.16} & 179.70  & 2288.5   \\
         \bottomrule
      \end{tabularx}

      \smallskip
      \scriptsize
      \begin{tablenotes}
         \RaggedRight
         \item[a] means collision-free and the violation of the curvature constraint is within 5\%.\\
         \item[b] Avg. time is larger than Table I since tasks here are much more challenging.
      \end{tablenotes}
   \end{threeparttable}
   \vspace{-1em}
\end{table}

\subsection{Quantitative Results}

1) \textit{Real-time Performance}:
To show the computational efficiency of our proposed method, we let the AV randomly navigate in Town03 of CARLA with dynamical agents and record the computation time for 20 minutes.
Static obstacles are randomly added to the AV's driving lane every $10\,\mathrm{s}$.
We also implement a non-incremental version of Algorithm \ref{alg:refinement} for comparison.
As shown in Table \ref{tab:perfomance}, our method runs above $20\,\mathrm{Hz}$ for most of the time.
The maximum value only occurs in some extreme circumstances, for example, avoiding an obstacle in a sharp turn.
The incremental version is nearly 2.5 times faster than the non-incremental version in the iterative refinement step.

2) \textit{Curvature Constraint}: To examine our proposed curvature constraint (\ref{eq:kappa_limit}), {\color{\mycolor}we randomly generate 1000 challenging planning tasks.} 
We implement two versions of the GP path planner -- (a) G3P: without curvature constraints, and (b) G3P-Cur: with curvature constraints.
Two open-source version of the SOTA planners are used for comparison: (c) DL-IAPS \cite{zhou2021dl_iaps}, an approach based on iterative path-speed refinement; {\color{\mycolor}(d) TDR-OBCA \cite{he2021tdr}, an approach based on optimal control.}
We feed DL-IAPS and TDR-OBCA with the result of G3P since they require a collision-free initial path as the warm start.
{\color{\mycolor}We adopt the curvature limit $0.2~\mathrm{m^{-1}}$ used in DL-IAPS.
The benchmark results are summarized in Table \ref{tab:my-table} and Fig. \ref{fig:curvature_constraint} shows an example.
From the results, we see that our method outperforms other SOTA planners both in success rate and efficiency.}

\begin{figure}
   \centering
   \vspace{0.5em}
   \subfloat[Curvature constraint]{\includegraphics[width=2.9in]{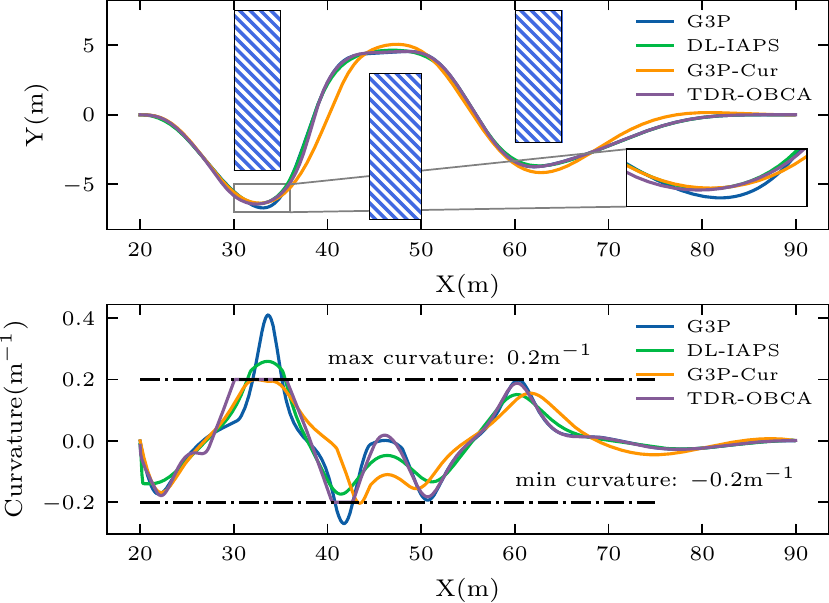}
      \label{fig:curvature_constraint}}\\
   \subfloat[Lateral acceleration constraint]{\includegraphics[width=2.9in]{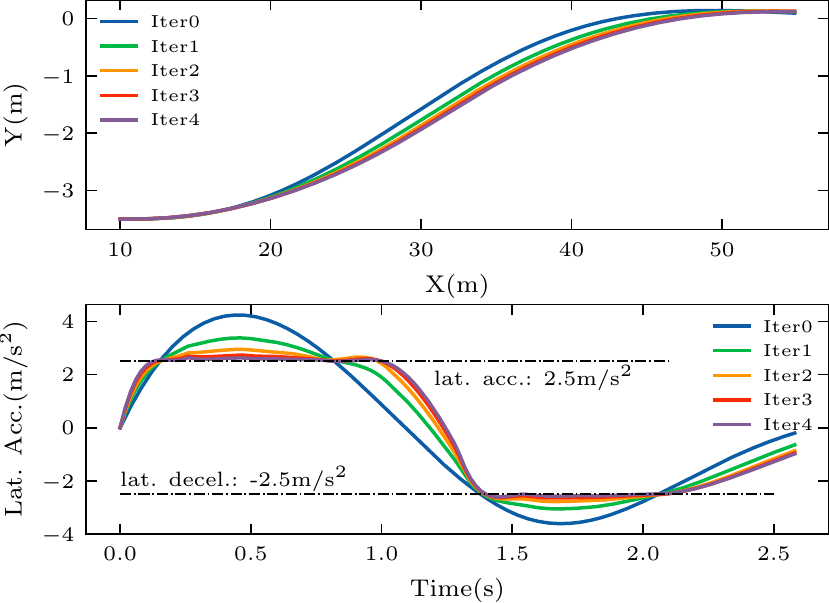}
      \label{fig:lat_acc_constraint}}
   \caption{
      (a) The AV has to avoid three random placed obstacles consecutively in a narrow space and should not violate the curvature constraint.
      (b) Resulting path and lateral acceleration in a lane changing scenario at a speed of 17.5 m/$\mathrm{s}^2$ after four iterations.}
   \label{fig:kinodynamic_constraint}
   \vspace{-1em}
\end{figure}

3) \textit{Lateral Acceleration Constraint}: Fig. \ref{fig:lat_acc_constraint} shows the results of Algorithm \ref{alg:refinement} in a lane change scenario at a speed of 17.5 $\mathrm{m/s^2}$ with a $3.5\,\mathrm{m}$ lane width.
The maximum lateral acceleration and deceleration are $2.5\,\mathrm{m/s^2}$ and $-2.5\,\mathrm{m/s^2}$ respectively.
The trajectory of iteration$\,0$ is not feasible, and it can be viewed equivalently to the result of the seminal work \cite{werling2012optimal}, as mentioned in \ref{sect:full_posterori}.
After four iterations, the lateral acceleration is limited to the acceptable range.

\section{CONCLUSIONS}

In this paper, we presented a real-time, kinodynamic trajectory planning framework in dynamic environments. 
We introduced a GP path planner that generates a collision-free path under curvature constraints, and outperforms the SOTA method.
An efficient s-t graph search method with a local truncation mechanism was also proposed to enable fast speed-profile generation.
Kinodynamic feasibility is guaranteed by the novel incremental refinement scheme.
Our method was examined with various numerical experiments.

The main limitation of this work is that the uncertainties are not well modeled (\textit{e.g.,} prediction uncertainty).
However, we believe that planning as probabilistic inference is a promising framework that can systematically deal with uncertainties, and this remains for future work.
\appendix
\subsection{White-noise on jerk model\label{wnoj}}
For (\ref{ltv-sde}) with WNOJM, we have $\bm{u}(s) = \mathbf{0}$ and 
\begin{align}
\mathbf{A}(s) \doteq \mathbf{A} = \begin{bmatrix} \mathbf{0} & \mathbf{I} & \mathbf{0} \\ \mathbf{0} & \mathbf{0} & \mathbf{I} \\ \mathbf{0} & \mathbf{0} & \mathbf{0} \end{bmatrix},
\mathbf{F}(s) \doteq \mathbf{F} = \begin{bmatrix} \mathbf{0} \\ \mathbf{0} \\ \mathbf{I} \end{bmatrix}.\label{eq:wnoj}
\end{align}
Let $\Delta s = s_{i} - s_{i-1}$, and the state transition matrix is computed by
\begin{equation}
\bm{\Phi}(s_{i}, s_{i - 1}) = \exp(\mathbf{A} \Delta s) = 
\begin{bmatrix}
\mathbf{I} & \Delta s\mathbf{I} & \frac{1}{2} \Delta s^2 \mathbf{I} \\
\mathbf{0}& \mathbf{I} & \Delta s \mathbf{I} \\
\mathbf{0}& \mathbf{0} & \mathbf{I}
\end{bmatrix}.\label{eq:phi}
\end{equation}
From (\ref{eq:Q_i}), we have
\begin{equation}
\begin{aligned}
\mathbf{Q}_i 
&= 
\begin{bmatrix}
\frac{1}{20} \Delta s^5 & \frac{1}{8} \Delta s^4 & \frac{1}{6}\Delta s^3\\
\frac{1}{8}\Delta s^4 & \frac{1}{3} \Delta s^3 & \frac{1}{2} \Delta s^2 \\
\frac{1}{6} \Delta s^3\, & \frac{1}{2} \Delta s^2 & \Delta s
\end{bmatrix} \otimes \mathbf{Q}_c \label{eq:Q}.
\end{aligned}
\end{equation}

\subsection{Proof of Theorem \ref{theo:jerk_optimal}\label{proof}}
\begin{proof}
As already proved by \cite{werling2012optimal}, we know that a quintic polynomial is the jerk optimal connection between two states $P_0 = [p_0,\dot{p}_0,\ddot{p}_0]$ and $P_1 = [p_1,\dot{p}_1,\ddot{p}_1]$.
Then we only have to prove that the result path given by (\ref{eq:interpolate}) and (\ref{eq:posterior}) does connect $\bar{\bm{d}}_0$ and $\bar{\bm{d}}_N$, and it is indeed a quintic polynomial.
By rearranging the posterior mean of (\ref{eq:posterior}), we have
\begin{equation}
\left( \bm{\mathcal{K}}^{ - 1} + \mathbf{M}^\top \mathbf{R}^{ - 1}\mathbf{M} \right)\bm{d}^* = \bm{\mathcal{K}}^{ - 1} \bm{\mu} + \mathbf{M}^\top \mathbf{R}^{ - 1} \bm{o},\label{eq:rerrange}
\end{equation}
Since $\bm{s}=\{s_0, s_N\}$, $\mathbf{M}$ becomes the identity matrix and (\ref{eq:rerrange}) simplifies to
\begin{equation}
\begin{aligned}
\left(\mathbf{R}\bm{\mathcal{K}}^{ - 1} + \mathbf{I} \right) \bm{d}^* &= \mathbf{R}\bm{\mathcal{K}}^{ - 1}\bm{\mu} + \bm{o}.
\end{aligned}
\end{equation}
Using $\mathbf{R}\to \bm{0}$, we have $\mathbf{R}\bm{\mathcal{K}}^{-1}\to \bm{0}$, which gives
\begin{equation}
\bm{d}^* = \bm{o} = \begin{bmatrix} \bm{\bar{d}}_0 & \bm{\bar{d}}_N \end{bmatrix}^\intercal.
\end{equation}
Next, we will prove that (\ref{eq:interpolate}) is quintic Hermite interpolation.
By using the property of the state transition matrix $\bm{\Psi}(\tilde{s}, s_0) = \bm{\Psi}(\tilde{s}, s_i)\bm{\Psi}(s_i, s_0)$ and (\ref{eq:wnoj}), (\ref{eq:interpolate}) simplifies to
\begin{equation}
\bm{d}(\tilde{s}) = \bm{\Lambda}(\tilde{s}) \bm{d}_i + \bm{\Psi}(\tilde{s}) \bm{d}_{i + 1}.\label{eq:hermite}
\end{equation}
Expanding it with (\ref{eq:phi}) and (\ref{eq:Q}), it is easy to verify that (\ref{eq:hermite}) is precisely the quintic Hermite interpolation; \textit{i.e.,} the result path is a quintic polynomial.
\end{proof}


\bibliographystyle{IEEEtran}
\bibliography{IEEEabrv,ref}

\end{document}